\documentclass[conference,letterpaper]{IEEEtran}

\usepackage[utf8]{inputenc} 
\usepackage[T1]{fontenc}
\usepackage{url}
\usepackage{ifthen}
\usepackage{cite}
\usepackage[cmex10]{amsmath}
\usepackage{booktabs}       
\usepackage[flushleft]{threeparttable}
\usepackage{amsfonts}       
\usepackage{nicefrac}       
\usepackage{microtype}      
\usepackage{amsthm,amsmath,amsfonts, amssymb}
\usepackage{multirow}
\usepackage{graphicx}
\usepackage{xcolor}
\usepackage{mathrsfs,mathtools}
\usepackage{bm, bbm}
\usepackage[algo2e]{algorithm2e}
\usepackage{wrapfig}
\usepackage{hyperref}

\RequirePackage{algorithm}
\RequirePackage{algorithmic}

\theoremstyle{plain}
\newtheorem{definition}{Definition}
\newtheorem{assumption}{Assumption}
\newtheorem{theorem}{Theorem}
\newtheorem{lemma}{Lemma}
\newtheorem*{remark}{Remark}

\DeclareMathOperator*{\esssup}{ess\,sup}

\newcommand{\x}{x}
\newcommand{\hats}{\widehat{s}}
\newcommand{\hatth}{\widehat{\theta}}
\newcommand{\E}{\mathbb{E}}

\newcommand{\alg}{\texttt{DSM-CUSUM}}

\begin{document}
\title{Sequential Change Point Detection via \\ Denoising Score Matching} 

\author{%
  \IEEEauthorblockN{Wenbin Zhou}
  \IEEEauthorblockA{Carnegie Mellon University\\
                    Pittsburgh, PA\\
                    Email: wenbinz2@andrew.cmu.edu}
  \and
  \IEEEauthorblockN{Liyan Xie}
  \IEEEauthorblockA{University of Minnesota\\
                    Minneapolis, MN\\
                    Email: liyanxie@umn.edu}
  \and
  \IEEEauthorblockN{Zhigang Peng}
  \IEEEauthorblockA{Georgia Institute of Technology\\
                    Atlanta, GA\\
                    Email: zpeng@gatech.edu}
  \and
  \IEEEauthorblockN{Shixiang Zhu}
  \IEEEauthorblockA{Carnegie Mellon University\\ 
                    Pittsburgh, PA\\
                    Email: shixiangzhu@cmu.edu}
}

\maketitle

\begin{abstract}
    Sequential change-point detection plays a critical role in numerous real-world applications, where timely identification of distributional shifts can greatly mitigate adverse outcomes.
    Classical methods commonly rely on parametric density assumptions of pre- and post-change distributions, limiting their effectiveness for high-dimensional, complex data streams. 
    This paper proposes a score-based CUSUM change-point detection, in which the score functions of the data distribution are estimated by injecting noise and applying denoising score matching. We consider both offline and online versions of score estimation. Through theoretical analysis, we demonstrate that denoising score matching can enhance detection power by effectively controlling the injected noise scale. Finally, we validate the practical efficacy of our method through numerical experiments on two synthetic datasets and a real-world earthquake precursor detection task, demonstrating its effectiveness in challenging scenarios.
\end{abstract}

\section{Introduction}

Sequential change point detection is a fundamental statistical problem that aims to identify the exact moment when the underlying distribution of a sequential data stream undergoes a structural change \cite{veeravalli2013quickest,siegmund1985sequential,tartakovsky2014sequential}. This capability is critical in various real-world applications where detecting and reacting to changes in real time can have significant implications. For example, change point detection is vital in financial markets to identify regime shifts \cite{barassi2020change, zhu2023sequential}, 
in manufacturing to detect system faults \cite{lai1995sequential},
and in epidemic control to detect the onset of outbreaks or shifts in infection trends \cite{zhu2022early}. 

Recent advances in sensor technology and signal processing have opened up new opportunities to uncover the intricate dynamics of change in various domains. 
For instance, in seismology, recent development and application of advanced earthquake detection techniques such as template matching and machine learning have resulted in an exponential growth in the quantity and quality of seismic records \cite{dong2023conditional, beroza2021machine}. These enriched datasets are critical for advancing our understanding of earthquake precursors \cite{pritchard2020new} -- an elusive yet transformative goal that could enable early warnings for catastrophic major earthquakes and significantly mitigate their impacts. While so far no reliable earthquake precursors have been identified \cite{picozza2021looking, conti2021critical, peng2024physical}, there is a renewed interest in studying them, largely driven by improved dense near-field geophysical observations, availability of big data and new development in machine learning methods in earthquake science \cite{dong2023conditional, beroza2021machine, mousavi2023machine}.

As data streams become increasingly complex and high-dimensional, there is a growing need for more robust and scalable change point detection frameworks \cite{liu2020unified}.
Many traditional techniques are grounded in modeling the log-likelihood of pre- and post-change distributions, which often involve parametric assumptions on density functions \cite{lai1998information, xie2023window}. While effective in certain scenarios, these assumptions can pose challenges in capturing the nuances of high-dimensional, interdependent data. This may limit their applicability in modern settings where distributional changes tend to be subtle, non-linear, and situated in complex data spaces that are not easily represented by simple parametric models \cite{li2015m, li2019scan}. 

In this paper, we propose a novel framework for sequential change point detection based on denoising score matching, a method that has demonstrated success in learning deep generative models \cite{song2020sliced,vincent2011connection}. The most salient feature of this approach is that it avoids parametric assumptions on the pre- and post-change distributions. Instead, we adopt a score network to learn the underlying distributions in a fully data-driven manner. To address the challenge of detecting subtle changes in regions of the data space where samples are scarce, we adopt a denoising strategy \cite{song2019generative}. By injecting controlled noise into the data, this strategy encourages the score network to explore underrepresented regions of the data distribution, improving its ability to detect changes. 
Furthermore, we develop a theoretical framework to quantify the tradeoff between the injected noise level and the score estimation accuracy. Our analysis reveals that careful selection of the noise level can minimize the error of the estimated distributional change and enable the score network to achieve robust detection performance.
Our framework is validated through experiments on both synthetic and real-world datasets, demonstrating its capability to detect changes in complex, high-dimensional data distributions. 

Our main contributions are summarized as follows: 
($i$) We propose a novel change point detection algorithm based on denoising score matching;
($ii$) We provide theoretical insights on the detection efficiency of the proposed method, and demonstrate the tradeoff on the level of injected noise;
($iii$) We validate the proposed method on both synthetic and real datasets, and demonstrate its superior performance against the state-of-the-art approaches.

\vspace{.05in}
\noindent{\bf Related work}.
Classical sequential change detection methods \cite{tartakovsky2014sequential,poor-hadj-QCD-book-2008,siegmund1985sequential, xie2021sequential} can be categorized into parametric and non-parametric approaches. Parametric methods estimate the parametric post-change distribution during the detection process, such as the generalized likelihood ratio (GLR) test \cite{page1954continuous, lai1998information} and window-limited tests \cite{xie2023window}. Nonparametric methods circumvent the need for explicit distributional assumptions and have been developed using techniques such as kernel density estimation \cite{liang2023quickest, sugiyama2012direct} and kernel-based distances \cite{li2015m, li2019scan}.

Recent advancements have introduced novel perspectives on change point detection. One notable development is the exact score-based detection method proposed by \cite{pmlr-v206-wu23b}, which utilizes the Hyv\"arinen score to detect changes in known unnormalized distributions, offering an alternative to the traditional log-likelihood approach. At the same time, neural network-based methods \cite{moustakides2019training, lee2023training} have broadened the applicability of change point detection to settings with unknown and complex distributions through likelihood-based CUSUM techniques. Building upon these advancements, our proposed method bridges the gap between score-based and neural network-based approaches. By effectively leveraging the flexibility of neural networks to parametrize unnormalized models, our approach achieves a robust and computationally efficient solution for detecting change points in complex, high-dimensional settings.

\section{Problem Setup}
\label{sec:setup}

Denote $\mathcal X$ as the (possibly high-dimensional) continuous sample space of the observed data. Let $p_0, p_1$ be two distinct\footnote{By ``distinct'' we mean that $p_0$ and $p_1$ must be different on a set in $\mathcal{X}$ with non-zero Lebesgue measure.} and unknown probability distributions on $\mathcal X$. Given data sequence $\{\x_t\in\mathcal X, t\in \mathbb{N}\}$, we suppose that the data-generating distribution changes from $p_0$ to $p_1$ at some unknown (but deterministic) change-point $\nu$, i.e., 
\begin{equation}
\label{eq:hypothesis}
    \begin{array}{ll}
    \x_t  \stackrel{\text{iid}}{\sim} p_0, &k = 1,2,\ldots,\nu-1, \\
    \x_t \stackrel{\text{iid}}{\sim} p_1, &k = \nu,\nu+1,\ldots
    \end{array} 
\end{equation}
The objective is to detect the unknown change-point $\nu$ as quickly as possible while controlling the false alarm rate. Such detection is usually performed by designing a stopping time $T$ on the observation sequence at which it is declared that a change has occurred \cite{xie2021sequential}.

The performance of a stopping time $T$ is measured by its \textit{average run length (ARL)} and \textit{worst-case average detection delay (WADD)}. The ARL measures the average time to false alarms and is defined as $\mathrm{ARL}:=\mathbb{E}_\infty[T]$, where $\E_\infty$ denotes the expectation under the probability measure when no change occurs. WADD is defined as 
\[
\mathrm{WADD}:= \underset{\nu \geq 1}{\operatorname{\sup}} \esssup \ \mathbb E_\nu\left[(T-\nu+1)^+| x_1,\ldots,x_{\nu-1}\right],
\]
where $\E_\nu$ is the expectation when the change-point equals $\nu$, and the essential supremum is taken over all possible change-points $\nu \geq 1$ and all possible pre-change samples $\{x_1,\ldots,x_{\nu-1}\}$. Our goal is to build a stopping time that can minimize WADD subject to a false alarm constraint $\mathrm{ARL}\geq \gamma$.

\section{Proposed Method}

We propose a novel change point detection algorithm based on denoising score matching (DSM).
We first present preliminary results on score matching. Then we describe our proposed algorithm, referred to as \alg. 

\subsection{Preliminaries}\label{sec:prelim}
We consider the scenario where the pre- and post-change density functions are either unavailable or difficult to estimate. This situation often arises, for example, in unnormalized models $p_\theta(\x) = \tilde p(\x,\theta)/Z_\theta$, where $Z_\theta$ is a normalizing constant that can be computationally intractable. Nonetheless, the distribution can still be characterized by the derivative of its log-density, $\nabla_x\log \tilde p(\x,\theta)$, known as the score function. Note that this score function does not depend on $Z_\theta$ and is often easier to estimate. Formally, we define the score and Hyv\"arinen score (abbreviated as H-score) as follows \cite{hyvarinen2005estimation}.
\begin{definition}
    \label{def:score}
    For a density function $p(\x)$, its (Stein) score function is defined as $s(\x)=\nabla_{\x} \log p(\x)$, and the 
    Hyv\"arinen score is defined as $H(\x ; s) = \operatorname{div} s(x) + \frac{1}{2} \Vert s(x)  \Vert_2^2$,
    where $\operatorname{div} s(\x) = \sum_{i = 1}^d \partial s_i(\x) / \partial \x_i$ is the divergence.
\end{definition}

We model the data distribution using a score function $s(\x;\theta)$ parameterized by $\theta$, which can be learned by minimizing the following Fisher divergence between the model and the true distribution $p(\x)$:
\begin{equation}
    \label{eq:fisher}
    \min_\theta D_{F} := \mathbb{E}_{\x \sim p} \| s(\x;\theta) - \nabla_{\x} \log p(\x) \|_2^2.
\end{equation}
This problem can be solved in practice using a family of methods known as score matching, which defines objectives that can be directly estimated from datasets sampled from $p$, without requiring the ground-truth data score \cite{hyvarinen2005estimation, song2019generative, song2020score}.

Such estimated score functions can be inaccurate in low-density regions due to insufficient training data for computing the score-matching objective. Noise perturbation mitigates this issue by adding noise to the data and training score-based models on the perturbed data \cite{vincent2011connection, song2019generative}. Given a transition kernel $\mathcal{K} (\x' | \x)$ that represents the noise injection process, the perturbed score model $s'(\x ; \theta)$ is learned by minimizing the following denoising score matching objective: 
\begin{equation}
    \label{eq:denoise}
    \mathbb{E}_{\x \sim p} \mathbb{E}_{\x' \sim \mathcal{K} (\cdot | \x) }  \left\| s'(\x' ; \theta) - \nabla_{\x'} \log \mathcal{K} (\x' | \x) \right\|_2^2.
\end{equation}
We choose the widely adopted Gaussian kernel $\mathcal{K} (\x' | \x) = \mathcal{N}(\x'; \x, \sigma^2 \mathbf{I})$, which corresponds to injecting Gaussian noise $\mathcal{N}(0, \sigma^2 \mathbf{I})$ to the data $\x$. This choice simplifies the expression for $\nabla_{\x'} \log \mathcal{K} (\x' | \x)=- ( \x' - \x ) / \sigma^2$ within \eqref{eq:denoise}. As shown in \cite{vincent2011connection}, the minimizer of \eqref{eq:denoise} converges to the true perturbed data score defined by the transition kernel.

\subsection{Proposed Method: Denoising Score Matching CUSUM}
\label{sec:alg}

Our algorithm contains two steps: \textit{training} and \textit{detection}. The training step uses the available reference data to estimate the unknown score function, while the detection steps uses the estimated score to calculate the online detection statistics.

We start with the case that reference datasets $\mathcal{D}_{\rm 0}$ and $\mathcal{D}_{\rm 1}$ are available and are {i.i.d.} sampled from pre- and post-change distributions, respectively. 
Following \eqref{eq:denoise}, we train two score models, both parameterized as $s(\x ; \theta)$ using neural networks represented by $\theta^{(i)}$ on $\mathcal{D}_i$ for $i=0,1$. We refer to this process as {\it offline score estimate}. The model parameter $\hatth^{(i)}$ is learned by minimizing the following denoising score matching objective computed on training data:
\begin{equation}
    \label{eq:train}
    \hatth^{(i)} \!:=\! \arg\min_{{\theta}} \!\! \sum_{\x_j \in \mathcal{D}_i}
    \sum_{k = 1}^K
    \left\| s\left(\x + \epsilon_{jk} ; {\theta} \right) +
    \frac{\epsilon_{jk}}{\sigma^2} \right\|_2^2,
    i = 0, 1.
\end{equation}
Here $K \geq 1$ is a hyperparameter that represents the number of repeated noise injections, and the injected noise $\epsilon_{jk}\overset{\text{iid}}{\sim}\mathcal{N}(0, \sigma^2 \mathbf{I})$. For simplicity we denote the resulted score models as $\hats_0(\x) = s(\x;\hatth^{(0)})$ and $\hats_1(\x) = s(\x;\hatth^{(1)})$.

Next, in the {detection phase}, for each time step $t$, we compute the difference of Hyvarinen scores of $\hats_0$ and $\hats_1$ evaluated at the data point $\x_t$, and cumulate it into a CUSUM type detection statistics:
\begin{equation}
    \label{eq:statistics}
    \mathcal S_{t} = \mathcal S_{t - 1}^+ + \Delta(\x_t), \,
    \Delta(\x_t) = H(\x_t ; \hats_0) - H(\x_t ; \hats_1), t\geq 1, 
\end{equation}
where $\mathcal S_{t - 1}^+ =\max\{\mathcal S_{t - 1},0\}$ and $\mathcal S_0=0$. Note that the calculation of the detection statistics relies solely on the estimated score function and does not depend on the full density function.

To perform the detection, we raise an alarm when the statistics $\mathcal S_t$ exceeds a pre-specified threshold $\tau$. Here $\tau$ is chosen to ensure that the average run length satisfies the desired lower bound, thereby controlling the false alarm rate. Formally, the stopping time for the detection can be defined as follows:
\begin{equation}
    \label{eq:stopping-time}
    T = \inf \left\{ t: \mathcal S_t \geq \tau \right\},
\end{equation}

We further extend our study to the general case where {\it no} reference data for the post-change distribution is provided beforehand, requiring the post-change score estimation in an online manner, which we call {\it online score estimate}. The proposed method can be easily adapted to such setting as follows: During the \textit{training phase}, we only learn the pre-change score model $\hat s_0$ using the reference data $\mathcal D_0$ via \eqref{eq:train}. 
During the {\it detection phase}, the post-change score model is estimated online using the most recent $w$ samples, where $w$ is the pre-defined window size. We initialize $\mathcal{S}_0 = \cdots = \mathcal{S}_w = 0$ and $\hatth_{0}=\cdots=\hatth_{w} = \widehat\theta^{(0)}$ (the pre-change score model parameter). For $t = w+1,\ldots$, we update the post-change score model $s(\x;\widehat{\theta}_t)$ by Gradient descent on the learning objective:
\begin{align}
    \label{eq:update}
    \hatth_t & \leftarrow \hatth_{t - 1} - \eta \nabla_{\hatth_{t - 1}}\mathcal{L}_{t}(\hatth_{t - 1}); \\
    \mathcal{L}_{t}(\theta) & \coloneqq
    \sum_{j = t - w + 1}^t
    \sum_{k = 1}^K
    \left\| s\left(\x_j + \epsilon_{jk} ; \theta \right) +
    \frac{\epsilon_{jk}}{\sigma^2} \right\|_2^2, \label{eq:w-loss}
\end{align}
where $\eta$ is the learning rate, and $\epsilon_{jk}\overset{iid}{\sim}\mathcal{N}(0, \sigma^2 \mathbf{I})$, and we use the subscripted notation $\widehat{\theta}_t$ to denote that the score model is being progressively updated. We note that \eqref{eq:update} can also be iterated for multiple gradient steps instead of a single step to improve the learning outcome, with the number of iterations typically determined by computational constraints.  

We update the detection statistic as defined in \eqref{eq:statistics} using the last updated score model $\hat s_1(x)=s(\x; \widehat{\theta}_{t-1})$, and the stopping time defined the same as \eqref{eq:stopping-time}. The detection procedure is summarized in Algorithm~\ref{alg:online-model}.

\setlength{\textfloatsep}{0.05in}
\begin{algorithm}[t!]
\caption{\alg~(Online)}
\label{alg:online-model}
\begin{algorithmic}[1]
    \REQUIRE Reference dataset
    $\mathcal{D}_0$, threshold $\tau$, window size $w$.
    \STATE Initialize $\nu = \infty$, $\mathcal{S}_0 = \ldots = \mathcal{S}_w = 0$, $t = w$.
    \STATE Train model $\hat s_0$ on dataset $\mathcal{D}_0$ and initialize $\hat s_1$ as $\hat s_0$.
    \WHILE{$S_t \leq \tau$}
        \STATE $t \leftarrow t + 1, \nu \leftarrow t$.
        \STATE Update post-change score model $\hat s_1$ based on the most recent $w$-length observations with \eqref{eq:update}.
        \STATE Compute test statistics $\mathcal{S}_t$ as in \eqref{eq:statistics}.
    \ENDWHILE
    \RETURN Detected change point $\nu$.
\end{algorithmic}
\end{algorithm}

\section{Theoretical Analysis}
\label{sec:theory}

This section provides the theoretical guarantees of the proposed algorithm in terms of worst-case average detection delay. We show that the injected noise level affects both the score model's estimation error and the detection efficiency, which provides insights into the trade-off involved in selecting the optimal noise level. The detailed proofs can be found in Appendix~\ref{app:proofs}.

Recall $\hats_0(\x)$ and $\hats_1(\x)$ are the trained pre- and post-change score model through denoising score matching, and $p_0,p_1$ are the pre- and post-change data distributions. Denote $p_{i, \sigma}(\x)=\int_{\mathcal X}p_i(y)\mathcal{N}(\x; y, \sigma^2 \mathbf{I})dy$, $i=0,1$ as the perturbed distribution after Gaussian noise injection.
We make the following assumption on the \textit{estimation error} $\epsilon_{\rm est}(\sigma)$ and the \textit{perturbation error} $\epsilon_{\rm pert}(\sigma)$. These constants can be explicitly derived (which may depend on $w$, $t$ and $K$) under certain conditions, see for example \cite{block2020generative}.
\begin{assumption}
    \label{ass:error}
    There exist constants $\delta\in[0,1]$, $\epsilon_{\rm est}(\sigma)\geq 0$, and $\epsilon_{\rm pert}(\sigma)\geq 0$ such that with probability at least $1 - \delta$, for $i = 0, 1$, the following conditions hold:
    \begin{align*}
        \mathbb{E}_{\x \sim p_{1}} \left\| \hats_i(\x) - \nabla \log p_{i, \sigma}( \x) \right\|_2 & \leq \epsilon_{\rm est}(\sigma), \\
        \mathbb{E}_{\x \sim p_{1}} \left\| \nabla \log p_{i, \sigma}( \x) - \nabla \log p_i( \x) \right\|_2 & \leq \epsilon_{\rm pert}(\sigma),
    \end{align*}
     where the randomness arises from the estimation of $\hats_0,\hats_1$ due to the randomness in training samples.
\end{assumption} 

\begin{assumption}
    \label{ass}
    Under Assumption~\ref{ass:error},
    we assume there exists some constants $\epsilon_{\rm div} \geq 0$
    and $C_1, C_2 \geq 1$, such that for $i = 0, 1$, the following conditions hold concurrently with Assumption~\ref{ass:error},
    \begin{align*}
        \operatorname{Var}_{\x \sim p_{1}} \Vert \widehat{s}_{i}(\x) - \nabla \log p_{i, \sigma}(\x) \Vert_2 & \leq C_1 \cdot \epsilon^2_{\rm est}(\sigma). \\
        \operatorname{Var}_{\x \sim p_{1}} \Vert \nabla \log p_{i, \sigma}(\x )  - \nabla \log p_{i}(\x) \Vert_2 & \leq C_2 \cdot \epsilon^2_{\rm pert}(\sigma). \\
        \mathbb{E}_{\x \sim p_1} \left| \mathrm{div} \widehat{s}_i (\x) - \mathrm{div} \nabla \log p_{i}(\x) \right| & \leq \epsilon_{\rm div}.
    \end{align*}
    The first and second-order moments of all score functions are jointly upper bounded by some $M < + \infty$.
\end{assumption}
The first two inequalities ensure that the errors exhibit bounded variation at a rate similar to their first-order moment (\textit{i.e.}, a finite coefficient of variance), while the last inequality requires the divergence of the score model to approximate the ground truth effectively. These conditions are satisfied when the model is well-specified and well-trained. 

\begin{lemma}
    \label{lem:error}
    Under Assumption~\ref{ass:error} and~\ref{ass}, define the distribution change error as
    \begin{align*}
        & \epsilon(\sigma) \coloneqq 2 M \left[ \epsilon_{\rm est}(\sigma) + \epsilon_{\rm pert}(\sigma) \right] + \epsilon_{\rm div} \\
        &\hspace{20pt} + 4M (M + 1) \cdot \max\left\{ C_1, C_2\right\} \cdot \left[ \epsilon^2_{\rm est}(\sigma) + \epsilon^2_{\rm pert}(\sigma) \right],
    \end{align*}    
    then with probability greater than $1 - \delta$,
    \begin{align*}
        \label{eq:stat-err}
        & \left| \mathbb{E}_{\x \sim p_1} [\Delta(\x)] - D_F(p_1 \| p_0) \right| \leq \epsilon(\sigma),
    \end{align*}
    where $D_F(p_1\| p_0)=\mathbb E_{\x\sim p_1} \|\nabla_{\x}\log p_1(\x)-\nabla_{\x}\log p_0(\x)\|_2^2$ is the fisher divergence.
\end{lemma}

Lemma~\ref{lem:error} establishes that the statistic computed using the plug-in estimate approximates the ground-truth Fisher divergence $D_F(p_1\| p_0)$ with an error bound of $\epsilon(\sigma)$, which scales proportionally to $\epsilon_{\rm est}(\sigma)$ and $\epsilon_{\rm pert}(\sigma)$. 
This observation paves the way for deriving an upper bound on the WADD of the detection procedure when the post-change score function is estimated offline using reference data $\mathcal D_1$, as detailed below.

\begin{theorem}[WADD with offline score estimate]
    \label{thm:offline}
    Under the offline score estimation setting, for a given threshold $\tau$, and under Assumption~\ref{ass:error} and~\ref{ass}, the following holds 
    with probability at least $1 - \delta$,
    $$
    \mathrm{WADD} \leq \frac{\tau}{  D_{F}\left( p_1 \Vert p_0 \right) - \epsilon(\sigma) }
    + \frac{ \mathbb{E}_{\x \sim p_1}[\Delta(\x)^2]}{ \left[ D_{F}\left( p_1 \Vert p_0 \right) - \epsilon(\sigma) \right]^2},
    $$
    where $\epsilon(\sigma)$ is defined in Lemma~\ref{lem:error}.
\end{theorem}

Theorem~\ref{thm:offline} shows that a larger error term $\epsilon(\sigma)$ results in a worse WADD upper bound. Note that when $\epsilon_{\rm est}(\sigma), \epsilon_{\rm pert}(\sigma), \delta \to 0$, which corresponds to a well-trained score model, Theorem~\ref{thm:offline}'s bound converges to the standard result in \cite{pmlr-v206-wu23b}. Also, as $\tau \to \infty$, the upper bound scales approximately to be $\mathcal{O}( \tau / \left[ D_{F}\left( p_1 \Vert p_\infty \right) - \epsilon(\sigma) \right])$. 
Under the same assumptions, we can also derive the upper bound of WADD under the online setting.

\begin{theorem}[WADD with online score estimate]
    \label{thm:online}
    Under the online score estimation setting, for the window size $w$ and threshold $\tau$, and under Assumption~\ref{ass:error} and~\ref{ass}, the following holds with probability at least $1 - \delta$,
    $$
    \mathrm{WADD} \leq
    \frac{\tau + c^{1/2}\tau^{1/2} +c   + c^{1/2}d^{1/2} + d}{D_F (p_1 || p_0 ) - \epsilon(\sigma)},
    $$
    where $\epsilon(\sigma)$ is defined in Lemma~\ref{lem:error} and
    \begin{align*}
        c & = \mathbb{E}_{\x \sim p_1}[\Delta(\x)^2] / \left[ D_F (p_1 || p_0 ) - \epsilon(\sigma) \right], \\
        d & = w \left[ D_F (p_1 || p_0 ) + \epsilon(\sigma) \right].
    \end{align*}
\end{theorem}

Compared to Theorem~\ref{thm:offline}, the windowing procedure introduces additional terms in the numerator, leading to an inflated upper bound. However, these terms become negligible as $\tau \to \infty$, indicating that when the false alarm rate must be kept low, the online algorithm offers theoretical guarantees comparable to the offline version, ensuring its reliable deployment. Nonetheless, when a sufficiently large offline reference dataset is available to ensure effective training of the score network, the offline setting remains theoretically superior.

It is evident from the Theorems that an optimal choice of $\sigma$ that minimizes $\epsilon(\sigma)$ will lead to a better detection performance, and by the definition of $\epsilon(\sigma)$, it depends on both the perturbation error $\epsilon_{\rm pert}(\sigma)$ and the estimation error $\epsilon_{\rm est}(\sigma)$. The following remark suggests that injecting noise generally enhances the accuracy of the statistic in capturing the distributional change:
\begin{remark} The process of noise injection involves a tradeoff: Higher noise levels increase perturbation error $\epsilon_{\rm pert}(\sigma)$ since noise injection alters the distribution of training data. However, they also reduce score estimation error $\epsilon_{\rm est}(\sigma)$ since noise injection improves score estimation. A key empirical observation is that $\epsilon_{\rm est}$ typically decreases at a much faster rate than $\epsilon_{\rm pert}$ increases. This trend is evident in the numerical simulation results presented in Figure~\ref{fig:tradeoff} (right). As a result, even a moderate level of noise injection can effectively lower the overall error, as demonstrated in Figure~\ref{fig:tradeoff} (left).
\end{remark}
\begin{figure}[t!]
    \centering    \includegraphics[width=0.95\linewidth]{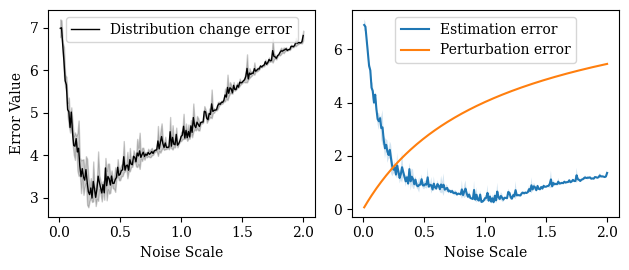}
    \vspace{-.15in}
    \caption{Simulation results on score estimation error $\epsilon_{\rm est}(\sigma)$, perturbation error $\epsilon_{\rm pert}(\sigma)$, and distribution change error $\epsilon(\sigma)$ (up to constant scaling factors) for varying values of $\sigma$. The ground-truth data is generated from a $1$D Gaussian mixture with two components.
    }
    \label{fig:tradeoff}
\end{figure}

\vspace{-0.07in}
\section{Numerical Experiments}
\label{sec:exp}

In this section, we present the numerical results of our method evaluated on both synthetic and real datasets. First, we compare the worst-case average detection delay (WADD) vs. the average run length (ARL) of our method against five baseline approaches on two synthetic datasets using offline estimation of the score model. Next, we demonstrate the application of our method to a real geophysical monitoring dataset, showcasing its ability to detect precursors to a major earthquake in an online setting. 
The score models in \alg~are implemented as single-layer feed-forward neural networks, trained for $2,000$ epochs until convergence with an injected noise scale of $1$. For the two synthetic data experiments and the real data experiment, the hidden dimensions are set to $2,048$, $512$, and $512$, respectively. In the real data experiment, a window size of $10$ is used for online score estimation.

\subsection{Synthetic Data Results}
\label{sec:syn}

\begin{figure}
    \centering
        \includegraphics[width=0.49\linewidth]{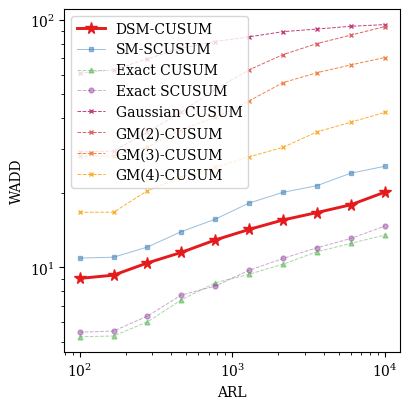}
        \includegraphics[width=0.49\linewidth]{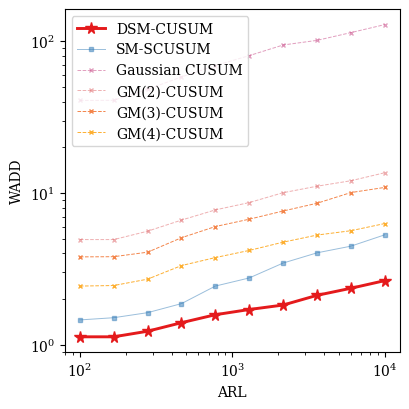}
    \vspace{-.3in}
    \caption{Comparison of WADD vs ARL on baseline methods for two synthetic datasets. Left: $2$D data by Gaussian mixtures. Right: $10$D data by deep nets.
    }
    \label{fig:synthetic}
\end{figure}

We construct two synthetic datasets with complex pre- and post-change distributions that cannot be easily captured by standard parametric models.
The first dataset is a $2$D ring-shaped distribution generated using Gaussian mixture models, while the second is a $10$D dataset obtained by applying nonlinear transformations, specified by random neural networks, to $4$D Gaussian data. 
We build the reference datasets in both scenarios to consist of $1,000$ datapoints.
Details of the data generation process are provided in Appendix~\ref{app:data}.

We evaluate our proposed \alg~under the offline score estimation setting, comparing to four baseline methods, including ($i$) CUSUM fitted with Gaussian mixture models with $n$ components (GM($n$)-CUSUM), ($ii$) an ablation model employing vanilla score matching (SM-SCUSUM), ($iii$) the Exact CUSUM \cite{page1954continuous}, and ($iv$) the exact Score CUSUM (Exact SCUSUM) \cite{pmlr-v206-wu23b}. We note that Exact CUSUM and Exact SCUSUM cannot be readily applied to the second synthetic dataset as the density function of the data, transformed by nonlinear mappings, is inaccessible. 
Detailed configurations are described in Appendix~\ref{app:baselines}.

Figure~\ref{fig:synthetic} presents the WADD vs. ARL tradeoff curves for all models. 
In both experiments, the proposed \alg~significantly outperforms other baselines, with its curves positioned substantially lower on the plot, making it the closest to those of the two oracle models. This underscores the robust modeling capabilities of unnormalized score-based models driven by deep neural networks, resulting in superior performance in detection tasks.
Lastly, note that the proposed \alg~outperforms the ablation model (SM-SCUSUM), despite both using the same neural network architecture. 
This performance gap also widens as the data distribution becomes more intricate. 
This further emphasizes the benefits of training with denoising score matching for handling more complex change point detection tasks.

\subsection{Real Data Results: Earthquake Precursor Detection}
\label{sec:real}

Earthquake precursor detection \cite{conti2021critical, picozza2021looking} focuses on identifying moments when the distribution of certain geophysical signals undergoes a shift, known as a ``precursor''. This shift provides critical insights for predicting the potential occurrence of major earthquakes. We evaluate the performance of our method on such a real-world dataset consisting of hourly high-resolution geophysical signals monitored prior to a moderate-size earthquake that occurred in China, 2014.

\begin{figure}[ht!]
    \centering
    \includegraphics[width=\linewidth]{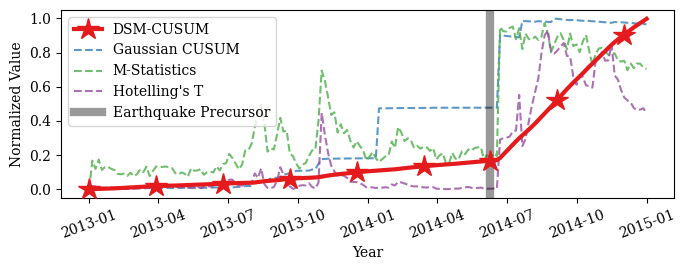}
    \vspace{-0.3in}
    \caption{
    Geophysical signal datasets: comparison of our method (red star) with three baseline methods (dashed lines). All values are normalized to the range $[0, 1]$ for ease of comparison. Domain experts identify the true precursor signals as emerging around early June 2014.
    }
    \label{fig:real}
\end{figure}

The dataset includes four types of signals from the three monitoring stations closest to the earthquake's epicenter over a two-year period. 
The signals include mid-layer water temperature, shallow-layer water temperature, static water level, and dynamic water level. After removing seasonal patterns, anomalies, and missing values, we use the first half of the trajectory as the reference dataset, and compare our algorithm under online score estimation with three baseline methods, including Hotelling's $T^2$, M-statistics \cite{li2015m}, and Gaussian CUSUM.
Baseline methods are detailed in Appendix~\ref{app:baselines}.

Figure~\ref{fig:real} illustrates the estimated statistics for all methods applied to the geophysical signal data stream. 
We observe that most baseline methods exhibit pronounced fluctuations between July 2013 and January 2014, which could lead to false alarms under lower threshold settings.
In contrast, our method demonstrates consistently low and stable statistics throughout this period. Furthermore, it accurately identifies the true precursor signal around mid-June 2014, which has been validated by domain experts. 
These results underscore the superior accuracy and reliability of our method in detecting earthquake precursors compared to the baseline approaches.

\section{Conclusion}

This paper addresses the key challenge of sequential change point detection where the pre- and post-change distributions are unknown. Our proposed method leverages the denoising score-matching objective to inject noise into the data and learn the score function of the data distribution, which is then used to construct a Hyv\"arinen score-based CUSUM statistic. Empirically, we demonstrate that injecting a small amount of noise into the data improves the score model’s ability to estimate the score function more accurately, which, in turn, enhances control over distributional changes and ultimately lowers the detection delay of the proposed statistics.
In addition, we validate our approach through numerical experiments on two synthetic datasets and one real-world geophysical monitoring dataset, showcasing its practical effectiveness in complex, potentially high-dimensional distributions.

A key limitation of our current work lies in the preliminary nature of the theoretical assumptions, which do not explicitly capture the precise relationship between noise scale and error values. We hypothesize that further investigation on score-matching learning theory could address this limitation, enabling the derivation of an optimal noise injection scale and facilitating a more systematic and effective implementation of our method.

\section*{Acknowledgement}
This research was supported by the NSF CAIG-2425888.

\newpage

\bibliographystyle{IEEEtran}
\bibliography{ref}

\newpage

\appendices

\section{Detailed Proofs for Section~\ref{sec:theory}}
\label{app:proofs}

We begin with introducing two useful technical lemmas (Lemma~\ref{lem:h-diff} and Lemma~\ref{lem:tech}).
Then we present the proof of Lemma~\ref{lem:error}, Theorem~\ref{thm:offline} and Theorem~\ref{thm:online} in sequence.
The pseudocode for the offline version of \alg~is also provided in Algorithm~\ref{alg:offline}.

For notation simplicity, we will use $\mathbb{E}_1$ instead of $\mathbb{E}_{\x \sim p_1}$ in our proofs.

\begin{lemma}
    \label{lem:h-diff}
    For two arbitrary distributions $p$ and $q$, we have
    \begin{align*}
        \mathbb{E}_{\x\sim p} [ H(\x ; q) - H(\x ; p) ] = D_F( p \Vert q ),
    \end{align*}
    where $D_F(\cdot \| \cdot)$ is the Fisher divergence between two distributions.
\end{lemma}

\begin{proof}
    Following \cite{hyvarinen2005estimation}, it has been shown that
    \begin{align*}
        D_F( p \Vert q ) & \coloneqq \mathbb{E}_p \left[ \| \nabla \log p(\x) - \nabla \log q(\x) \|_2^2 \right] \\
        & = \mathbb{E}_p \left[ \frac{1}{2} \| \nabla \log p(\x) \|_2^2 + H(\x ; q) \right],
    \end{align*}
    therefore
    \begin{align*}
        & \mathbb{E}_p [ H(\x ; q) - H(\x ; p) ] \\
        = & D_F(p \| q) -  \mathbb{E}_p \left[ \frac{1}{2} \| \nabla \log p(\x) \|_2^2 \right] \\
        & \underbrace{- D_F(p \| p)}_{= 0} + \mathbb{E}_p \left[ \frac{1}{2} \| \nabla \log p(\x) \|_2^2 \right] \\
        = & D_F(p \| q).
    \end{align*}
\end{proof}

\begin{lemma}
    \label{lem:tech}
    For two score function $\nabla \log p$ and $\nabla \log q$, define the following two terms \textit{w.r.t.} the random variable $\x$,
    \begin{align*}
        A(\x) & = \left\| \nabla \log p(\x) -  \nabla \log q(\x) \right\|_2, \\
        B(\x) & = \left\| \nabla \log p(\x) +  \nabla \log q(\x) \right\|_2.
    \end{align*}
    Then for a distribution $\pi$ over $\mathcal{X}$, there is
    \begin{align*}
         & \frac{1}{2} \mathbb{E}_{\pi} \left[  \left\| \nabla \log p(\x) \right\|_2^2 - \left\| \nabla \log q(\x) \right\|_2^2 \right] \\
         \leq & \frac{1}{2} \mathbb{E}_{\pi} [A(\x)] \mathbb{E}_{\pi} [B(\x)] + \frac{1}{2} \operatorname{Var}_{\pi} [A(\x)] \operatorname{Var}_{\pi} [B(\x)].
    \end{align*}
\end{lemma}

\begin{proof}
    Starting from the left-hand side, we have
    \allowdisplaybreaks
    \begin{align*}
         & \frac{1}{2} \mathbb{E}_{\pi} \left[  \left\| \nabla \log p(\x) \right\|_2^2 - \left\| \nabla \log q(\x) \right\|_2^2 \right] \\
         = & \frac{1}{2} \mathbb{E}_{\pi} \left[ \nabla \log p(\x)^\top \nabla \log p(\x) - \nabla \log q(\x)^\top \nabla \log q(\x)\right] \\
         = & \frac{1}{2} \mathbb{E}_{\pi} \langle \nabla \log p(\x) + \nabla \log q(\x), \nabla \log p(\x) - \nabla \log p(\x) \rangle \\
         \leq & \frac{1}{2} \mathbb{E}_{\pi} \left[ A(x) B(x) \right] \\
         \leq & \frac{1}{2} \mathbb{E}_{\pi} [ A(x) ]  \mathbb{E}_{\pi} [ B(x) ] + \frac{1}{2} \operatorname{Cov}_\pi \left( A(x), B(x) \right) \\
         \leq & \frac{1}{2} \mathbb{E}_{\pi} [A(\x)] \mathbb{E}_{\pi} [B(\x)] + \frac{1}{2} \operatorname{Var}_{\pi} [A(\x)] \operatorname{Var}_{\pi} [B(\x)].
    \end{align*}
    Note that the first and third inequality both follow from the Cauchy-Schwartz inequality.
\end{proof}

\begin{algorithm}[t]
\caption{\alg~(Offline)}
\label{alg:offline}
\begin{algorithmic}[1]
    \REQUIRE Reference datasets
    $\mathcal{D}_0, \mathcal{D}_1$, threshold $\tau$.
    \STATE Train the two score models $\hat{s}_0, \hat{s}_1$ using the reference datasets by \eqref{eq:train}. 
    \WHILE{$S_t \leq \tau$}
        \STATE $t \leftarrow t + 1, \nu \leftarrow t$.
        \STATE Compute test statistics $\mathcal{S}_t$ as in \eqref{eq:statistics}.
    \ENDWHILE
    \RETURN Detected change point $\nu$
\end{algorithmic}
\end{algorithm}

Next, we provide the complete proof for Lemma~\ref{lem:error}.

\begin{proof}[\bf Proof of Lemma~\ref{lem:error}]
    Given the definition of the proposed statistics in \eqref{eq:statistics}, we can make the following decomposition,
    \begin{align*}
        \mathbb{E}_1 [\Delta(\x)] \\
        & \hspace{-.5in} = \mathbb{E}_1 \left[ H(\x ; \widehat{p}_0) - H(\x ; p_{0, \sigma}) \right] + \mathbb{E}_1 \left[ H(\x ; p_{0, \sigma}) - H(\x ; p_{0}) \right] \\
        & \hspace{-.5in} + \mathbb{E}_1 \left[ H(\x ; p_{1, \sigma}) - H(\x ; \widehat{p}_1) \right] + \mathbb{E}_1 \left[ H(\x ; p_{1}) - H(\x ;  p_{1, \sigma}) \right] \\
        & \hspace{-.5in} + \mathbb{E}_1 \left[ H(\x ; p_{0}) - H(\x ; p_{1}) \right].
    \end{align*}
    By Assumption~\ref{ass}, we can derive that, with probability greater than $1 - \delta$,
    \begin{align*}
        \mathbb{E}_1 \| \widehat{s}_0(\x) - \nabla \log p_{0, \sigma}(\x) \|_2 & \leq \epsilon_{\rm est}(\sigma), \\
        \mathbb{E}_1 \| \widehat{s}_0(\x) + \nabla \log p_{0, \sigma}(\x) \|_2 & \leq 2M, \\
        {\rm Var}_1 \| \widehat{s}_0(\x) - \nabla \log p_{0, \sigma}(\x) \|_2 & \leq C_1 \epsilon^2_{\rm est}(\sigma),\\
        {\rm Var}_1 \| \widehat{s}_0(\x) + \nabla \log p_{0, \sigma}(\x) \|_2 & \leq 4M(M+1).
    \end{align*}
    Therefore by Lemma~\ref{lem:tech},
    \begin{align*}
        & \mathbb{E}_1 \left[ H(\x ; \widehat{p}_0) - H(\x ; p_{0, \sigma}) \right]  \\
        \leq & 
        M \epsilon_{\rm est}(\sigma) + 2M(M+1) C_1 \epsilon^2_{\rm est}(\sigma).
    \end{align*}
    Following similar procedures, we can derive the upper bound for the second, third and fourth term to obtain bounds of similar forms. For the fifth term, we can use Lemma~\ref{lem:h-diff} and get
    $$
    \mathbb{E}_1 \left[ H(\x ; p_{0}) - H(\x ; p_{1}) \right] = D_F( p_1 \Vert p_0 ).
    $$
    Combining all previous derivations, and by the definition of $\epsilon_{\rm div}$ in Assumption~\ref{ass}, we can get with probability greater than $1 - \delta$,
    \begin{align*}
        & \left| \mathbb{E}_1 [\Delta(\x)] - D_F(p_1 \| p_0) \right| \\
        \leq & 2 M \left[ \epsilon_{\rm est}(\sigma) + \epsilon_{\rm pert}(\sigma) \right] + \epsilon_{\rm div} \\
        + & 4M (M + 1) \cdot \max\left\{ C_1, C_2\right\} \cdot \left[ \epsilon^2_{\rm est}(\sigma) + \epsilon^2_{\rm pert}(\sigma) \right].
    \end{align*}
    This completes the proof.
\end{proof}

Before we state the proof of Theorem~\ref{thm:offline} and Theorem~\ref{thm:online}, we introduce the following statistics of random walk:
\begin{equation}
    \label{eq:proxy}
    \mathcal{S}_t' = \sum_{i = 1}^t{\Delta_i}, \quad \forall t = 1, \ldots, N.
\end{equation}
Then we define the corresponding stopping time as $T' := \inf \left\{ t: \mathcal S_t' \geq \tau \right\}$ given some threshold $\tau$.
It has been proved in Lemma 4 and Theorem 1 of \cite{xie2023window} that the following upper-bound relationship holds:
\begin{equation}
    \label{eq:bound-chain}
    \mathrm{WADD} \leq \mathbb{E}_1 \left[ T \right] \leq \mathbb{E}_1 \left[ T' \right],
\end{equation}
therefore, we only need to upper-bound the expected stopping time of the proxy statistics \eqref{eq:proxy} in order to upper-bound the desired quantities in Theorem~\ref{thm:offline} and Theorem~\ref{thm:online}.

\begin{proof}[\bf Proof of Theorem~\ref{thm:offline}]
    Since $\mathcal{D}_0$ and $\mathcal{D}_1$ are independent, therefore under $\x \stackrel{i.i.d.}{\sim} p_1$, the statistics $\Delta(\x
    )$ are also \textit{i.i.d.} samples. Hence using \eqref{eq:bound-chain} and by Wald's identity:
    \begin{equation}
        \label{eq:wald}
        \mathbb{E}_1[T'] = \frac{ \mathbb{E}_1 [\mathcal{S}_{T'}'] } {  \mathbb{E}_1[\Delta(\x)]}.
    \end{equation}
    The numerator can be upper-bounded as
    \begin{align*}
        \mathbb{E}_1 [\mathcal{S}_{T'}'] & = \tau + \mathbb{E}_1 [ \mathcal{S}_{T'}' - \tau ] \\
        & \leq \tau + \sup_{\tau \geq 0} \mathbb{E}_1 [ \mathcal{S}_{T'}' - \tau ] \\
        & \leq \tau + \frac{\mathbb{E}_1 [ \left( \max \left\{ \Delta(\x), 0\right\} \right)^2 ]}{\mathbb{E}_1 [\Delta] } \\
        & \leq \tau + \frac{ \left( \mathbb{E}_1 [ \Delta(\x) ] \right)^2 + \mathrm{Var}_1[\Delta(\x)]}{\mathbb{E}_1 [\Delta(\x)] },
    \end{align*}
    where the second inequality follows from \cite{lorden1970excess}.
    Plug back to \eqref{eq:wald} we get:
    $$
    \mathbb{E}_1[T'] \leq \frac{\tau}{\mathbb{E}_1[\Delta(\x)]} + \frac{\mathrm{Var}_1[\Delta(\x)]}{\left( \mathbb{E}_1 [ \Delta(\x) ] \right)^2 } + 1.
    $$
    Using the simplified notation $\epsilon(\sigma)$ introduced in the main text, we have established in Lemma~\ref{lem:error} that
    \begin{align*}
        \left| \mathbb{E}_1 [ \Delta(\x) ] - D_F (p_1 \| p_0 ) \right| \leq \epsilon(\sigma), \quad w.p. \geq 1 - \delta,
    \end{align*}
    therefore combining all previous derivations, we know that with probability greater than $1 - \delta$,
    \begin{align*}
        \mathrm{WADD} & \leq \frac{\tau}{D_{F}\left( p_1 \Vert p_0 \right) - \epsilon(\sigma) }
        + \frac{ \mathrm{Var}_1[\Delta(\x)]}{ \left[ D_{F}\left( p_1 \Vert p_0 \right) - \epsilon(\sigma) \right]^2} + 1 \\
        & \frac{\tau}{  D_{F}\left( p_1 \Vert p_0 \right) - \epsilon(\sigma) }
        + \frac{ \mathbb{E}_1[\Delta(\x)^2]}{ \left[ D_{F}\left( p_1 \Vert p_0 \right) - \epsilon(\sigma) \right]^2}.
    \end{align*}
    Since $\mathbb{E}_1[\Delta(\x)^2] < +\infty$, therefore the RHS is finite. This finishes the proof.
\end{proof}

\begin{proof}[\bf Proof of Theorem~\ref{thm:online}]
    Again, throughout our derivation, we condition on the reference dataset.
    Notice that the statistics $\Delta(\x_t)$ is now a $w$-dependent sequence \cite{xie2023window}, starting from the RHS of \eqref{eq:bound-chain},
    \begin{align*}
        \mathbb{E}_1 [\mathcal{S}_{T'}']
        & \coloneqq \mathbb{E}_1 \left[ \sum_{t = w + 1}^{T'} \Delta(\x_t) \right] \\
        & = \mathbb{E}_1 \left[ \sum_{t = w + 1}^{T' + w} \Delta(\x_t) \right] - \mathbb{E}_1 \left[ \sum_{t = T' + 1}^{T' + w} \Delta(\x_t) \right].
    \end{align*}
    Define filtration $\mathscr{F}_{t}$ as the $\sigma$-field generated by all samples up to time $t$, we can decompose the first term as,
    \begin{align*}
        & \mathbb{E}_1 \left[ \sum_{t = w + 1}^{T' + w} \Delta(\x_t) \right] \\
        = & \mathbb{E}_1 \left[ \sum_{t = w + 1}^{\infty} \Delta(\x_t) \mathbbm{1}_{\{ T' \geq t - w\}} \right] \\
        = & \mathbb{E}_1 \left[ \sum_{t = w + 1}^{\infty} \mathbb{E}_1 [ \Delta(\x_t) | \mathscr{F}_{t - 1}] \mathbbm{1}_{\{ T' \geq t - w\}} \right] \\
        \geq & \left[ D_F (p_1 || p_0 ) - \epsilon(\sigma) \right] \cdot \mathbb{E}_1 [T'], \quad w.p. \geq 1 - \delta.
    \end{align*}
The second equality uses the tower property of expectations and $\mathbbm{1}_{\{ T' \geq t - w\}}$ is $\mathscr{F}_{t - 1}$-measurable.
The inequality uses the fact that since the $\sigma$-field generated by $\Delta$ is contained in $\mathscr{F}_{t - 1}$ during online estimation, therefore, by Lemma~\ref{lem:error},
\begin{equation}
    \label{eq:online-wp-error}
    \left| \mathbb{E}_1 [\Delta(\x_t) | \mathscr{F}_{t - 1} ]
     - D_F (p_1 || p_0 ) \right| \leq \epsilon(\sigma), \quad w.p. \geq 1 - \delta.
\end{equation}

Similarly, for the second term, we can derive that,
\begin{align*}
    & \mathbb{E}_1 \left[ \sum_{t = T' + 1}^{T' + w} \Delta(\x_t) \right] \\
    = & \mathbb{E}_1 \left[ \sum_{t = w + 1}^{\infty} \Delta(\x_t) \mathbbm{1}_{\{ T' < t\}} \mathbbm{1}_{\{ T' \geq t - w\}} \right] \\
    = & \mathbb{E}_1 \left[ \sum_{t = w + 1}^{\infty} \mathbb{E}_1 [\Delta(\x_t) | \mathscr{F}_{t - 1} ] \mathbbm{1}_{\{ T' < t\}} \mathbbm{1}_{\{ T' \geq t - w\}} \right] \\
    \leq & \left[ D_F (p_1 || p_\infty ) + \epsilon(\sigma) \right] \cdot w, \quad w.p. \geq 1 - \delta.
\end{align*}
Again, for the second equality, we are using the fact that $\mathbbm{1}_{\{ T' < t\}}$ and $\mathbbm{1}_{\{ T' \geq t - w\}}$ are $\mathscr{F}_{t - 1}$-measurable. For the inequality, we use \eqref{eq:online-wp-error}.
Combining the two results and plug into the original expression we get
$$
\mathbb{E}_1 [ T'] \leq \frac{\mathbb{E}_1 [ \mathcal{S}_{T'}' - \tau] + \tau + w \left[ D_F (p_1 || p_\infty ) + \epsilon(\sigma) \right]}{D_F (p_1 || p_0 ) - \epsilon(\sigma)},
$$
where $\tau$ is the threshold value.
Denote
\begin{align*}
    c & = \mathbb{E}_1[\Delta(\x)^2] / \left[ D_F (p_1 || p_0 ) - \epsilon(\sigma) \right], \\
    d & = w \left[ D_F (p_1 || p_0 ) + \epsilon(\sigma) \right].
\end{align*}
Following similar derivations of bounding the overshoot term in Theorem 1 of \cite{xie2023window}, we can derive that 
\begin{align*}
    \mathbb{E}_1 [ \mathcal{S}_{T'}' - \tau]
    \leq c + c^{1/2} \tau^{1/2} + c^{1/2} d^{1/2},
\end{align*}
Plugging this result into the right-hand side, then with probability greater than $1 - \delta$,
$$
\mathrm{WADD} \leq
\frac{c + c^{1/2}\tau^{1/2} + \tau + c^{1/2}d^{1/2} + d}{D_F (p_1 || p_0 ) - \epsilon(\sigma)}.
$$
Note that $\mathbb{E}[\Delta(\x)^2] < \infty$, therefore the upper bound is finite. We have finished the proof.
\end{proof}

\section{Additional Experiment Details}
\label{app:exp}

In this section, we provide further details on the experimental setup described in Section~\ref{sec:exp}.

\subsection{Threshold Simulation}
The threshold \(\tau\) is typically chosen to balance the trade-off between the probabilities of false alarms and successful detection, as determined by a target ARL value. To reduce computational effort in determining thresholds for large target ARL values, we employ an efficient approximation algorithm. This algorithm leverages the fact that the stopping time \(T\) under the pre-change regime approximately follows an exponential distribution when the ARL is large. Such approximation methods are widely utilized in online change detection.

The high-level idea of the procedure is that, instead of simulating the mean of the distribution of $T:= \inf\{t: \mathcal S_t \geq \tau\}$ directly, we obtained an estimate of the mean from an estimate of the cumulative distribution function of $T$ based on $N_1$ iterations.
Specifically, in each iteration, we simulate the pre-change trajectory with $N_2$ time steps, and compute the maximum of the detection statistics at these $N_2$ time steps. These maximum values under $N_1$ iterations are then denoted as $W_{1,\text{max}}, W_{2,\text{max}}, \ldots, W_{N_1,\text{max}}$. For the desired ARL values $\gamma$, we approximate the stopping time $T$ as an exponential distribution with mean $\gamma$. Thus we have $P(W_{\text{max}} < \tau ) = P(T > N_2)  \approx e^{-N_2/\gamma}$. Thus the corresponding threshold $\tau$ can be approximated as the $e^{-N_2/\gamma}$ quantile of the set \(\{W_{1,\text{max}}, W_{2,\text{max}}, \ldots, W_{N_1,\text{max}}\}\). 
We use $N_1=200$, and $N_2=1000$. Note that we can also use more iterations and longer sequences within each iteration, which tends to improve the approximation accuracy.

\begin{figure}
    \centering
    \includegraphics[width=0.8\linewidth]{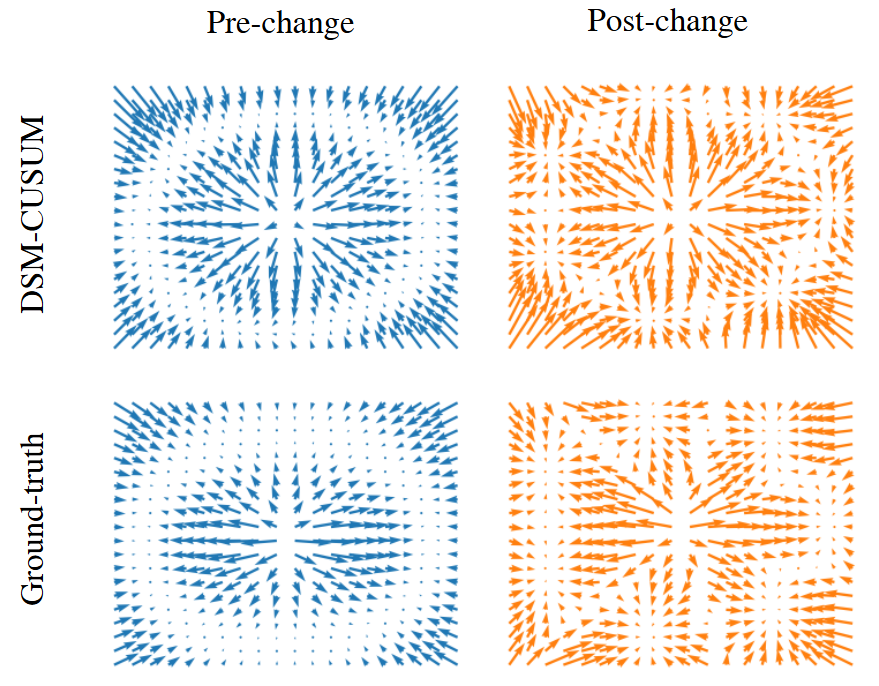}
    \caption{Comparison of the estimated score function obtained using \alg~against the ground-truth for the 2D Gaussian mixture dataset. The results demonstrate that the score models effectively capture the ground-truth score function.}
    \label{fig:score-comp}
\end{figure}

\begin{figure}
    \centering
    \includegraphics[width=0.45\linewidth]{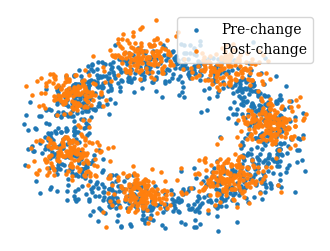}
    \includegraphics[width=0.45\linewidth]{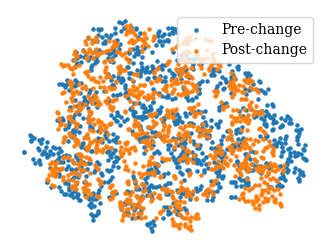}
    \caption{Visualization of the two datasets. The left column depicts $2$D Gaussian mixture data, while the right column represents $10$D neural network data (visualized using t-SNE).}
    \label{fig:dataset}
\end{figure}

\subsection{Synthetic Data}
\label{app:data}
The two synthetic datasets used in Section~\ref{sec:syn} are defined as follows: 

\paragraph{2D Gaussian mixture data}
Both pre- and post-change distributions are specified by $2$-dimensional Gaussian mixture models, where the pre-change distribution has $30$ components, and the post-change distribution has $8$ components, each of the means is evenly distributed around the circle centered at the origin with a radius of $3$. The ground-truth pre- and post-change distributions are specified as:
\begin{align*}
    p_1 & = \sum_{i = 1}^3 \frac{1}{30} \mathcal{N}\left(
    \begin{bmatrix}
        8 \cos \frac{i}{15}\pi - \frac{1}{2} \\
        8 \sin \frac{i}{15}\pi - \frac{1}{2}
    \end{bmatrix},
    \begin{bmatrix}
        1 & 0 \\
        0 & 1
    \end{bmatrix}
    \right);\\
    p_\infty & = \sum_{i = 1}^3 \frac{1}{8} \mathcal{N}\left(
    \begin{bmatrix}
        8 \cos \frac{i}{4}\pi + \frac{1}{2} \\
        8 \sin \frac{i}{4}\pi + \frac{1}{2}
    \end{bmatrix},
    \begin{bmatrix}
        1 & 0 \\
        0 & 1
    \end{bmatrix}
    \right).
\end{align*}
With the Gaussian mixture model, the ground-truth data distribution and score functions are readily available, enabling the deployment of the exact CUSUM and SCUSUM methods as described in \cite{pmlr-v206-wu23b}.
We introduce a small $\frac{1}{2}$-size offset to prevent the data from being indistinguishable for detection methods based on first and second moments.

\paragraph{10D Neural Network data}
We first generate standard Gaussian data in $4$D space, then transform the data using two sets of neural networks, whose weights and biases are randomly initialized according to standard Gaussian. To make it more challenging, we additionally deploy a mean and covariance matrix matching algorithm to make the data challenging and indistinguishable. The loss function is denoted by 
$$
\min_\theta \| \widehat{\mu}_1 - \widehat{\mu}_\infty \|_2^2 + \| \widehat{\Sigma}_1 - \widehat{\Sigma}_\infty \|_2^2,
$$
where $\theta$ denotes the parameters of the two neural networks, and $\widehat{\mu}_j, \widehat{\Sigma}_j, j = 1, \infty$ denotes the sample mean and covariance matrix of the generated pre-change and post-change data.

\subsection{Real Data \& Result Implications}

The objective of earthquake precursor detection is to detect any relevant earthquake signals before the earthquake starts (\textit{.e.g.}, three months in advance). 
The final cleaned data contains $17,521$ data points, covering a two-year time interval starting from January 01, 2013, to January 01, 2015.

Our results on the real data have two key practical implications:
($i$) This experiment retrospectively validates the judgment of domain experts and provides a robust foundation for further seismological investigations into the mechanism underlying this precursor signal, advancing earthquake prediction research.
($ii$) The experiment highlights the potential generalizability of our method for detecting precursor signals to predict earthquakes. Early detection---\textit{e.g.}, up to three months in advance---provides local governments and residents sufficient time to prepare and evacuate, significantly reducing potential losses. This underscores the method's high value in disaster mitigation and other high-stakes scenarios.

\subsection{Baseline Models}
\label{app:baselines}

Our baseline models fall in the broad category of the widely used cumulative sum (CUSUM) test statistic \cite{page1954continuous}, which is also the optimal test, defined as:
\begin{equation}
    \label{eq:cusum}
    S_t = \max_{0 \leq k \leq t} \sum_{j = k + 1}^t \log \frac{p_1 (\x_j)}{p_0(\x_j)} = S_{t - 1}^+ + \log \frac{p_1 (\x_t)}{p_0(\x_t)}.  
\end{equation}
The stopping rule is defined by $T = \inf \left\{ t: \mathcal S_t \geq \tau \right\}.$ When $p_0$ and $p_1$ are known, we refer to \eqref{eq:cusum} as the exact CUSUM test.

\paragraph{Gaussian CUSUM}
The pre-change and post-change distribution are parameterized as two multivariate Gaussian PDFs $\mathcal{N}(\mu_0, \Sigma_0)$ and $\mathcal{N}(\mu_1, \Sigma_1)$. The log-likelihood used in \eqref{eq:cusum} can be simplified as:
\begin{align*}
    \log \frac{\mathcal{N}(\mu_1, \Sigma_1)}{\mathcal{N}(\mu_0, \Sigma_0)} =  & \frac{1}{2} \log \left( \frac{|\Sigma_0|}{|\Sigma_1|} \right) + \frac{1}{2} \left[ (x - \mu_0)^\top \Sigma_0^{-1} (x - \mu_0) \right] \\
    & - \frac{1}{2} \left[ (x - \mu_1)^\top \Sigma_1^{-1} (x - \mu_1) \right].
\end{align*}
The parameters $\mu_0, \mu_1, \Sigma_0, \Sigma_1$ are fitted via standard maximum likelihood estimation (MLE).

\paragraph{Gaussian Mixture CUSUM}
The pre-change and post-change distributions are parameterized as two Gaussian mixture distributions with $n$ components,
$$
p_0 = \sum_{i=1}^n \pi_i^{(0)} \mathcal{N}\left(\mu_i^{(0)}, \Sigma_i^{(0)}\right),
\quad
p_1 = \sum_{i=1}^n \pi_i^{(1)} \mathcal{N}\left(\mu_i^{(1)}, \Sigma_i^{(1)}\right).
$$
The parameters $\pi_i^{(0)}, \mu_i^{(0)}, \Sigma_i^{(0)}, \pi_i^{(1)}, \mu_i^{(1)}, \Sigma_i^{(1)}$ are fitted using the expectation-maximization (EM) algorithm.

\section{Additional Backgrounds}

We provide more advanced background information related to our proposed method.

\vspace{0.1in}
{\bf \noindent Score-based Models.}
Denoising score matching is later generalized in SMLD \cite{song2019generative} and Score SDEs \cite{song2020score} to extend the framework of diffusion models \cite{ho2020denoising}. Specifically, when there are $n >1$ noise levels $\left\{\sigma_i \right\}_{i = 1}^n$ and allowing for the noise level as a conditional input to the parameterized score functions $s(\x, \sigma_i)$ to create a continuous ($n \to \infty$) or discrete (finite $n$) diffusion process
\begin{equation*}
    \mathbb{E}_i \mathbb{E}_{p(\x)} \mathbb{E}_{q_\sigma (\x^{(i)} | \x)}  \left[ \lambda_i \left\Vert s(\x^{(i)}, \sigma_i) - \nabla_{\x^{(i)}} \log \mathcal{K}_{\sigma_i} (\x^{(i)} | \x) \right\Vert \right].
\end{equation*}
where $\mathcal{K}_{\sigma_i}(\cdot | \cdot)$ denotes a known transition kernel parameterized by noise scale $\sigma_i$. 
$\lambda_i$ are coefficients used to scale the value of errors and can typically be chosen as $\lambda_i \propto \sigma_i^2 $
Empirically, when the training process converges well, for all $\sigma_i$, there is $s(\x, \sigma_i) \to \nabla_\x \log p_{\sigma_i}(\x)$.

In the meantime, to make score matching scalable in high-dimension settings, sliced score matching was proposed in 
\cite{song2020sliced}. Denote random projection direction $\x \in \mathbb{R}^d$, and their distribution $P_\x$. We draw $\mathbf{v}_1, \ldots, \mathbf{v}_M \stackrel{i.i.d.}{\sim} P_\mathbf{v}$ and compute,
\begin{equation*}
    \mathcal{L}(\mathbf{v}) \coloneqq \frac{1}{T} \sum_{t = 1}^T \sum_{m = 1}^M \left[ \mathbf{v}_m^\top J s (\x_t) \mathbf{v}_m + \frac{1}{2} \left( \mathbf{v}_m^\top s(\x_t) \right) ^2 \right],
\end{equation*}
where $J$ denotes the Jacobian matrix operator.
Typically $P_\mathbf{v}$ can be specified as simple distributions such as standard multivariate Gaussian distribution, and $M = 1$ suffice in practice.

\end{document}